%% file: main.tex
\documentclass[10pt,journal,compsoc]{IEEEtran}
\ifCLASSOPTIONcompsoc
  \usepackage[nocompress]{cite}
\else
  \usepackage{cite}
\fi
\ifCLASSINFOpdf
\else
\fi
\usepackage[caption=false,font=footnotesize,labelfont=sf,textfont=sf]{subfig}
\usepackage[caption=false,font=footnotesize]{subfig}
\usepackage{amssymb,amsmath,amsthm}
\usepackage{graphicx}
\usepackage{algorithm}
\usepackage{algorithmic}
\usepackage{bm}
\usepackage{multirow}
\usepackage{booktabs}
\usepackage{makecell}
\newcommand{\cK}{{\cal K}}
\def\reals{\mathbb{R}}
\def\alphas{\bm{\alpha}}
\def\betas{\bm{\beta}}
\def\omegas{\bm{\omega}}
\def\nus{\bm{\nu}}
\newtheorem{theorem}{Proposition}

\begin{document}
%
\title{Learning with Asymmetric Kernels: Least Squares and Feature Interpretation}
\author{Mingzhen~He,
        Fan~He,
        Lei~Shi,
        Xiaolin~Huang,~\IEEEmembership{Senior~Member,~IEEE}
        and~Johan~A.K.~Suykens,~\IEEEmembership{Fellow,~IEEE}
\IEEEcompsocitemizethanks{
\IEEEcompsocthanksitem M. He, F. He and X. Huang (corresponding author) are with Institute of Image Processing and Pattern Recognition, Shanghai Jiao Tong University, Shanghai
200240, P.R. China (e-mail: \{mingzhen\_he; hf-inspire; xiaolinhuang\}@sjtu.edu.cn).\protect\\
\IEEEcompsocthanksitem L. Shi is with the Shanghai Key Laboratory for Contemporary Applied Mathematics, Fudan University, Shanghai 200433, China, and also with the School of Mathematical Sciences, Fudan University, Shanghai 200433, China (e-mail: leishi@fudan.edu.cn).\protect\\
\IEEEcompsocthanksitem J. A. K. Suykens is with the Department of Electrical Engineering (ESAT-STADIUS), KU Leuven, B-3001 Leuven, Belgium (e-mail:
johan.suykens@esat.kuleuven.be).\protect\\}
\thanks{}}
\IEEEtitleabstractindextext{%
\begin{abstract}
Asymmetric kernels naturally exist in real life, e.g., 
for conditional probability and directed graphs. However, most of the existing kernel-based learning methods
require kernels to be symmetric, 
which prevents the use of asymmetric kernels. This paper addresses the asymmetric kernel-based learning in the framework of the least squares support vector machine named \emph{AsK-LS}, resulting in the first classification method that can utilize asymmetric kernels directly. We will show that AsK-LS can learn with asymmetric features, namely source and target features, while the kernel trick remains applicable, i.e., the source and target features exist but are not necessarily known.  
Besides, 
the computational burden of AsK-LS is as 
cheap as dealing with symmetric kernels. Experimental results on the Corel database, directed graphs, and the UCI database will show that, in the case asymmetric information is crucial, the proposed AsK-LS can learn with asymmetric kernels and performs much better than the existing kernel methods that have to do symmetrization to accommodate asymmetric kernels. 
\end{abstract}

}

\maketitle

\IEEEdisplaynontitleabstractindextext

%
\IEEEpeerreviewmaketitle

\IEEEraisesectionheading{\section{Introduction}\label{sec:intro}}

%
%
%
%
\IEEEPARstart{K}{ernel-based} learning \cite{scholkopf2002learning, suykens2002least, vapnik2013nature} is an important scheme in machine learning and has been widely used in classification \cite{cortes1995support,soman2009machine}, regression \cite{drucker1997support}, clustering \cite{girolami2002mercer}, and many other tasks. 
Traditionally, a kernel that can be used in kernel-based learning should satisfy Mercer's condition \cite{jeffreys1999methods}. For Mercer's condition, there are two well-known requirements on the kernel $\cK(\cdot,\cdot): \mathbb{R}^{d}\times\mathbb{R}^{d} \mapsto \mathbb{R}$: for samples $\{\bm{x}_i, y_i\}_{i=1}^m$, where $d$ and $m$ are the dimension and the number of data. The kernel matrix $\bm{K}: \bm{K}_{ij}= \cK(\bm{x}_i, \bm{x}_j)$ should be i) symmetric and ii) positive semi-definite (PSD). When the latter condition is relaxed, the flexibility is enhanced and those methods are called \textit{indefinite learning}, for which some interesting results 
could be found in \cite{ong2004learning,haasdonk2005feature,huang2017indefinite,oglic2019scalable,liu2021fast}. 

However, discussion on relaxing the symmetry condition is rare. Many asymmetric similarities exist in real life. For example, in directed graphs, the adjacency matrix is certainly asymmetric and thus the connection similarity is asymmetric, i.e., $d(\bm{x}_i, \bm{x}_j) \neq d(\bm{x}_j, \bm{x}_i)$: the path from the $i$-th node to the $j$-th node is not equal to that from the $j$-th node to the $i$-th node. Conditional probability, which has been widely used to measure the directional similarity \cite{hinton2002stochastic}, is also asymmetric: $p(\bm{x}_i|\bm{x}_j) \neq p(\bm{x}_j|\bm{x}_i)$. Those asymmetric measurements cannot be used in the current kernel-based learning directly. Let us consider the support vector machine (SVM, \cite{cortes1995support}):
\begin{equation}\label{svm}
\min_{\alphas\in \mathbb{R}^n} ~~ \frac{1}{2} \alphas^\top \bm{K} \alphas - {\bf 1}^\top \alphas, \mathrm{~s.t.~} \bm{Y}^\top \alphas  = 0, 0 \leq \alphas \leq C\bm{1}, 
\end{equation}
where $C > 0$ is a pre-given constant, $\bm{1}$ is a n-dimensional vector of all ones, $\bm{Y} = [y_1,\cdots,y_m]^\top$, and $\alphas$ is a dual variable vector. When $\bm{K}$ is asymmetric, at least in (\ref{svm}), we can directly use it. However, by noticing that \[
\alphas^\top \bm{K} \alphas = \alphas^\top \frac{\bm{K}^\top + \bm{K}}{2} \alphas, \forall \alphas\in \mathbb{R}^n, \bm{K} \in \mathbb{R}^{n\times n},
\]
one may find that only the symmetric part of an asymmetric kernel is learned by directly using it in the SVM.

Another popular kernel-based learning framework is the least squares support vector machine (LS-SVM, \cite{suykens1999least, suykens2002least}). Its dual form is 
the following linear system,
\begin{equation*}\label{LS-SVM}
\left[
\begin{array}{cc}
0 &\bm{Y}^\top\\ 
\bm{Y}&\frac{\bm{I}}{\gamma}+\bm{H}\\ 
\end{array}
\right ]
\left[
\begin{array}{c}
 b \\ 
 \alphas\\
\end{array}
\right ]
=
\left[
\begin{array}{c}
 0\\
\bm{1}\\
\end{array}
\right ],
\end{equation*}
where $\bm{I}$ is an identity matrix and $\bm{H}:\bm{H}_{ij} = y_i \bm{K}_{ij} y_j$, respectively. An interesting point here is that using an asymmetric kernel in the LS-SVM will not reduce to its symmetrization and asymmetric information can be learned. Then we can develop asymmetric kernels in the LS-SVM framework in a straightforward way. The corresponding kernel trick, the feature interpretation, and the asymmetric information will be investigated in this paper. Notice that we do not claim that asymmetric kernels could not be applied in the SVM, but it is not straightforward and requires further investigation. Similarly, for symmetric but indefinite kernels, the solving method in the LS-SVM framework keeps easy  \cite{huang2017indefinite}, while the SVM needs delicately design in form, theory, and solving-algorithm \cite{loosli2015learning,oglic2019scalable}.

\begin{figure*}[ht!]
\centering
\setlength{\abovecaptionskip}{2pt}%
\setlength{\belowcaptionskip}{3pt}%
\includegraphics[width=145mm]{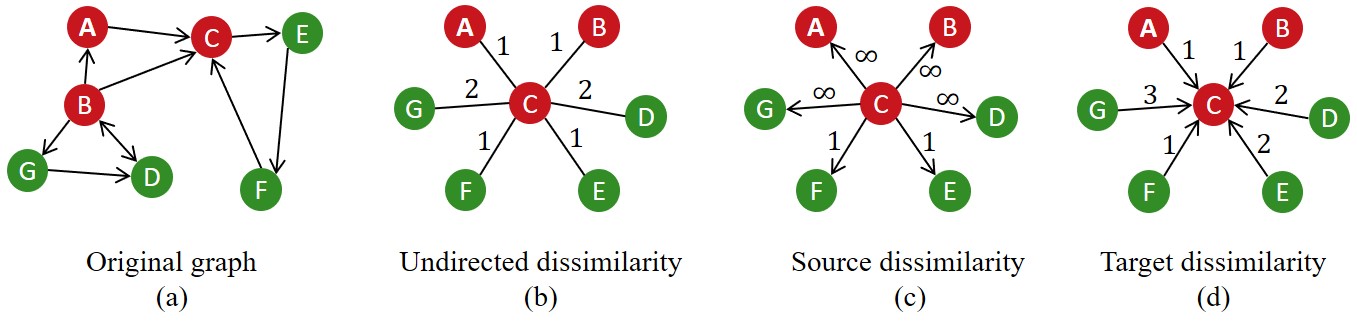}
\caption{A simple illustration for asymmetric dissimilarity. In order to show the asymmetric information, the dissimilarity between sample C and the other samples is shown in sub-figures (b), (c) and (d) as an example. (a) There are seven samples on a directed graph where red and green colors indicate two categories. The dissimilarity from one sample to another one is defined as the shortest path from the first sample to the latter, for example, the dissimilarity from sample A to sample E is 2 and if the sample can not be reached, dissimilarity is $\infty$. (b) Symmetrization. The directed edges are replaced by undirected edges. (c) Source space. The dissimilarity from C to others. (d) Target space. The dissimilarity from other samples to C.
}
\label{fig:1}
\end{figure*}

In this paper, a novel method called \emph{AsK-LS} for learning with asymmetric kernels in the framework of least squares will be established.
The most important discussion is to investigate the kernel trick and the feature interpretation on how the asymmetric information could be extracted by Ask-LS. Generally, there are two features involved in the kernel trick. Using the concept in directed graphs, which have two feature embeddings for source and target nodes \cite{ou2016asymmetric, zhou2017scalable, khosla2019node}, we call the features in AsK-LS as \emph{source feature} and \emph{target feature}. The name distinguishes two features but does not mean that it can only be used for directed graphs. For the singular value decomposition, asymmetric kernels could be introduced into the LS-SVM framework \cite{suykens2016svd} where the two features are related to columns and rows of the matrix, respectively.



In the rest of this paper, we will first discuss asymmetric kernels and illustrate that there are two different features embedded in an asymmetric kernel; see Section 2 for details. Then in Section \ref{sec:method} we will formulate AsK-LS, discuss its feature interpretation, and design the solving algorithm. In principle, an asymmetric kernel contains more information than the symmetric one. Thus, the proposed AsK-LS will demonstrate advantages when the asymmetric information is crucial, as numerically evaluated in Section \ref{sec:experiment}. Section \ref{sec:conclusion} will end this paper with a brief conclusion.


\section{Asymmetric kernels}
In the classical kernel-based learning, the kernel $\cK$ is symmetric, i.e., the kernel matrix $\bm{K}_{ij} = \cK(\bm{x}_i, \bm{x}_j)$ for training samples is also symmetric. But there could be many asymmetric kernels. For example, in image classification tasks, Kullback-Leibler (KL, \cite{moreno2003kullback}) divergence could be used to measure the dissimilarity between two probability distributions. The directed graph is another example, where the similarity between two nodes is essentially asymmetric.


Intuitively, $\bm{K}$ being asymmetric contains more information than that being symmetric. For symmetric kernels, the kernel trick (there are additional conditions for the existence of kernel trick; see, e.g, \cite{jeffreys1999methods}) means that there is a feature mapping $\phi$ such that $\cK(\bm{u},\bm{v}) = \left<\phi(\bm{u}), \phi(\bm{v})\right>$ for two samples $u$ and $v$. Then it is expected that there are more features for asymmetric kernels. Fig. \ref{fig:1} illustrates a simple case for asymmetric dissimilarity. In Fig. \ref{fig:1}(b-d), we illustrate three methods which make the kernel matrix symmetric. For source dissimilarity, one can extract a nonlinear feature mapping, denoted as $\phi_s(\bm{u})$. Meanwhile, a target nonlinear feature mapping, denoted as $\phi_t(\bm{u})$ which is generally different from $\phi_s(\bm{u})$, can be also extracted. However, in the existing kernel-based learning, only symmetric kernels are acceptable and hence one has to use: symmetric similarity, for example, $(\bm{K}^\top + \bm{K})/2$ or $\bm{K}^\top \bm{K}$, which indicates that those symmetrization methods may lose the asymmetric information. 




Besides KL divergence and directed graphs, there are other tasks where the asymmetric kernels may be superior. In kernel density estimation problems, asymmetric kernels performed better than symmetric ones when the underlying random variables were bounded \cite{mackenzie2004asymmetric,kuruwita2010density,koul2013large}. In Gaussian process regression tasks, Pintea et al. argued that it was helpful to set an individual kernel parameter for each data center, which enabled each data center to learn a proper kernel parameter in its neighborhood and resulted in an asymmetric kernel matrix \cite{pintea2018asymmetric}. In federated learning tasks \cite{huang2021fl}, an asymmetric neural tangent kernel was established to address the issue that the gradient of the global machine was not determined by local gradient directions directly.

Our aim in this paper is to propose a novel method to directly learn with asymmetric kernels and correspondingly can learn with two feature mappings. For a long time, symmetrization is the main way for dealing with asymmetric kernels. In an early paper \cite{tsuda1999support}, Tsuda let the asymmetric kernel matrix $\bm{S}$ be symmetric by multiplying its transpose, then a new symmetric matrix $\bm{Q}$ was obtained as $\bm{Q}=\bm{S}^\top\bm{S}$. Munoz et al. utilized a pick-out method to convert the asymmetric kernel into the symmetric one \cite{munoz2003support}. Moreno et al. studied the KL divergence kernel $D(P,Q)$ in the SVM on multimedia data \cite{moreno2003kullback}. They defined $D(P,Q) =\rm KL(P||Q) + KL(Q||P)$ to satisfy the Mercer condition, but the asymmetric information was disappeared. Koide and Yamashita proposed an asymmetric kernel method and applied it to the Fisher's discriminant (AKFD) \cite{koide2006asymmetric}. They claimed that an asymmetric kernel $\mathcal{K}(\bm{x},\bm{y}) = \left<\phi_1(\bm{x}),\phi_2(\bm{y}) \right>$ was generated by the inner product between two different feature mappings. In the AKFD, the decision function was assumed to be spanned by $\{\phi_1(\bm{x}_i)\}_{i=1}^m$ and input data were mapped by $\phi_2$. However, the assumption of the AKFD was very strict and the situation that the decision function was spanned by $\{\phi_2(\bm{x}_i)\}_{i=1}^m$ was not considered. Wu et al. proposed a hyper asymmetric kernel method to learn with asymmetric kernels between data from two different input spaces such as query space $\mathcal{X}$ and document space $\mathcal{Y}$ \cite{wu2010asymmetric}, while an asymmetric kernel degenerated to a symmetric one when two spaces were identical, i.e., $\mathcal{X}=\mathcal{Y}$. In summary, these works used symmetrization methods at the optimization level, which canceled the asymmetric information and was not expected in the asymmetric kernel-based learning.

It was interesting that the matrix singular value decomposition (SVD) could be merged in the LS-SVM framework \cite{suykens2016svd,suykens2017deep}. The matrix to be decomposed could be asymmetric and even non-square, implying that the LS-SVM could tolerate asymmetric kernels. From the viewpoint of the LS-SVM setting, the matrix SVD was related to two feature maps acting on the column vectors and the row vectors of the matrix, respectively. For directed graphs, it was also possible to use the adjacency matrix without the label to extract embeddings as the source and target features, respectively \cite{ou2016asymmetric, zhou2017scalable, khosla2019node}. These works, although in an unsupervised setting, demonstrated that asymmetric kernels can be studied rather than through the symmetrization process. 
 

\section{Asymmetric kernels in the LS-SVM}
\label{sec:method}

\subsection{PSD and indefinite kernels in the LS-SVM}
\label{sec:ls_svm}

Given training samples $\{\bm{x}_i,y_i\}_{i=1}^m$ with $\bm{x}\in \mathbb{R}^d$ and $y\in \{+1,-1\}$, a discriminant function $f:\mathbb{R}^d \rightarrow \mathbb{R}$ is constructed to classify the input samples. For linearly inseparable problems, a non-linear feature mapping $\phi:\mathbb{R}^d\rightarrow \mathbb{R}^p$ is needed, where $\mathbb{R}^p$ is a high-dimensional space.

LS-SVMs with PSD kernels can be solved by the following optimization problem \cite{suykens1999least},
\begin{equation}
\begin{aligned}\label{ls-svm}
&\min_{\omegas , b, \xi} \quad \frac{1}{2} \omegas^\top\omegas + \frac{\gamma}{2}\sum_{i=1}^m\xi^2_i\\
&\begin{array}{r@{\quad}r@{}l@{\quad}l}
\mbox{s.t.} \quad y_i(\omegas^\top \phi(\bm{x}_i)+b) = 1-\xi_i \\
\forall i \in\{1,2,\cdots,m\},
\end{array}
\end{aligned}
\end{equation}
where the discriminant function $f$ is formulated as $f(\bm{x}) = \omegas^ \top \phi(\bm{x})+b$. When the kernel is generalized to a non-PSD one, the primal problem is as follows \cite{huang2017indefinite},
\begin{equation}
\begin{aligned}\label{indefinite-ls-svm}
&\min_{\omegas , b, \xi} \quad \frac{1}{2} (\omegas_+^\top\omegas_+-\omegas_-^\top\omegas_-) + \frac{\gamma}{2}\sum_{i=1}^m\xi_i^2\\
&\begin{array}{r@{\quad}r@{}l@{\quad}l}
\mbox{s.t.} \quad y_i(\omegas_+^\top \phi_+(\bm{x}_i) + \omegas_-^\top \phi_-(\bm{x}_i) +b) = 1-\xi_i \\
\forall i \in\{1,2,\cdots,m\},
\end{array}
\end{aligned}
\end{equation}
where $\phi_+:\mathbb{R}^d\rightarrow \mathbb{R}^{p_1}$ and $\phi_-:\mathbb{R}^d\rightarrow \mathbb{R}^{p_2}$ are two non-linear feature mappings, both $\mathbb{R}^{p_1}$ and $\mathbb{R}^{p_2}$ are potential high-dimensional spaces. The discriminant function here is formulated as $f(\bm{x}) = \omegas_+^\top \phi_+(\bm{x}_i) + \omegas_-^\top \phi_-(\bm{x}_i) +b$.

Although the feature interpretations of (\ref{ls-svm}) and (\ref{indefinite-ls-svm}) are not the same, their dual problems share the same formulation as below,
\begin{equation}\label{linear_system_ls_svm}
\left[
\begin{array}{cc}
0 &\bm{Y}^\top\\ 
\bm{Y}&\frac{\bm{I}}{\gamma}+\bm{H}\\ 
\end{array}
\right ]
\left[
\begin{array}{c}
 b \\ 
 \alphas\\
\end{array}
\right ]
=
\left[
\begin{array}{c}
 0\\
\bm{1}\\
\end{array}
\right ].
\end{equation}


The kernel trick, which gives the feature interpretation of (\ref{linear_system_ls_svm}), is different for different types of kernels. If the kernel is PSD then, 
\[
\mathcal{K}(\bm{x}_i,\bm{x}_j) =  \left<\phi(\bm{x}_i),\phi(\bm{x}_j)\right>,
\]
If the kernel is non-PSD but has a positive decomposition, for which the conceptual condition and a practice judgment can be found in \cite{liu2021fast}, the kernel trick becomes,
\begin{equation*}
    \begin{aligned}
        \mathcal{K}(\bm{x}_i,\bm{x}_j) &= \left<\phi_+(\bm{x}_i),\phi_+(\bm{x}_j)\right>-\left<\phi_-(\bm{x}_i),\phi_-(\bm{x}_j)\right>\\
        &= \mathcal{K}_+(\bm{x}_i,\bm{x}_j)-\mathcal{K}_-(\bm{x}_i,\bm{x}_j),
    \end{aligned}
\end{equation*}
where $\mathcal{K}_+$ and $\mathcal{K}_-$ are two PSD kernels.

\subsection{AsK-LS}

When looking at the framework of LS-SVM from the viewpoint of solving (\ref{linear_system_ls_svm}), there is not any problem if $\bm{K}$ is asymmetric. It is still well-defined and a solution can be readily obtained. The key problem is to analyze what we really learn if $\bm{K}$ is asymmetric. 


First, we define a generalized kernel trick to present a kernel as an inner product of two mappings $\phi_s$ and $\phi_t$.
\newtheorem{myDef}{Definition}
\begin{myDef}\label{def:asy_kernel}
	A kernel trick for a kernel $\mathcal{K}: \reals^{d_s} \times \reals^{d_t} \rightarrow \reals$ 
	can be defined by the inner product of two different feature mappings as follows:
    \[
    \mathcal{K}(\bm{u},\bm{v}) = \left<\phi_s(\bm{u}),\phi_t(\bm{v})\right>, \forall \bm{u} \in \mathbb{R}^{d_s}, \bm{v} \in \mathbb{R}^{d_t},
    \]
    where $\phi_s:\reals^{d_s} \rightarrow \mathbb{R}^p$, $\phi_t:\reals^{d_t} \rightarrow \mathbb{R}^p$, 
    and $\mathbb{R}^p$ is a high-dimensional even infinite-dimensional space.
\end{myDef}

Different from the classical kernel trick, the above definition allows different $\phi_s$ and $\phi_t$, of which even the dimension $d_s$ and $d_t$ could be different. In this paper, we consider the case $d:=d_s=d_t$ then both $\mathcal{K}(\bm{u},\bm{v})$ and $\mathcal{K}(\bm{v},\bm{u})$ are well-defined and the kernel matrix for training data is square but asymmetric. 
Definition \ref{def:asy_kernel} is compatible with the existing symmetric kernels, including PSD and indefinite ones.
\begin{enumerate}
    \item The symmetric and positive semi-definite kernel $\mathcal{K}(u,v)$ can be defined as follows:
    \[
    \mathcal{K}(\bm{u},\bm{v}) = \left<\phi(\bm{u}),\phi(\bm{v})\right>, \forall \bm{u},\bm{v} \in \reals^d,
    \]
    in the situation when, two feature mappings $\phi_s$ and $\phi_t$ are identical $\phi:=\phi_s=\phi_t$. $\phi_s,\phi_t\in\reals^d\rightarrow\mathbb{R}^p$.
    \item The symmetric and indefinite kernel $\mathcal{K}(u,v)$ can be defined as follows:
\begin{equation*}
    \begin{aligned}
         \mathcal{K}(\bm{u},\bm{v}) &= \mathcal{K}_1(\bm{u},\bm{v}) - \mathcal{K}_2(\bm{u},\bm{v})
         \\
         & = \left<\phi_1(\bm{u}),\phi_1(\bm{v})\right> - \left<\phi_2(\bm{u}),\phi_2(\bm{v})\right>\\
         & = \left[
         \begin{array}{cc}
              \phi_1(\bm{u})  \\
              \phi_2(\bm{u})
         \end{array}\right]^\top
         \left[
         \begin{array}{cc}
              \phi_1(\bm{v})  \\
              -\phi_2(\bm{v})
         \end{array}\right]\\
         & := \left<\phi_s(\bm{u}),\phi_t(\bm{v})\right>, \forall \bm{u},\bm{v} \in \reals^d,\\
    \end{aligned}
\end{equation*}
in the situation when two feature mappings $\phi_s$ and $\phi_t$ are not identical, $\mathcal{K}_1$ and $\mathcal{K}_2$ are two PSD kernels and $\phi_1:\reals^d\rightarrow\mathbb{R}^{p_1}$ and $\phi_2:\reals^d\rightarrow\mathbb{R}^{p_2}$ are two high-dimensional feature mappings corresponding to $\mathcal{K}_1$ and $\mathcal{K}_2$, respectively. $\mathbb{R}^{p_1}$ and $\mathbb{R}^{p_2}$ are two high-dimensional spaces.
\end{enumerate}

The kernel trick associated with an asymmetric kernel contains two different feature mappings. Using the concept from directed graphs, we call them \emph{source} and \emph{target} features, respectively. Then for each sample, e.g., a node in a directed graph, we can extract two features from different views and classify the sample in the framework of the least squares support vector machine, which is hence called \emph{Ask-LS}. The AsK-LS in the primal space takes the following form, 
\begin{equation}
\begin{aligned}\label{asy_ls_svm}
&\min_{\omegas, \nus, b_1,b_2,\bm{e},\bm{h}} \quad \omegas^\top \nus + \frac{\gamma}{2}\sum_{i=1}^m e_i^2 + \frac{\gamma}{2}\sum_{i=1}^m h_i^2 \\
&\begin{array}{r@{\quad}r@{}l@{\quad}l}
\mbox{s.t.}\quad y_i(\omegas^\top \phi_s(\bm{x}_i)+b_1) = 1-e_i \\
\quad y_i(\nus^\top \phi_t(\bm{x}_i)+b_2) = 1-h_i \\
\quad \quad\quad\quad\quad\quad\quad \forall i \in\{1,2,\cdots,m\},
\end{array}
\end{aligned}
\end{equation}
where $e_i,h_i \in \reals$ and $\gamma$ is the regularization coefficient of misclassification errors. In AsK-LS (\ref{asy_ls_svm}), the sample plays different roles and can have different meanings in different tasks. In the matrix decomposition, the features could be given as column and row vectors, which could lead to an asymmetric kernel for unsupervised learning \cite{suykens2016svd}. 

Now let us investigate the dual problem of (\ref{asy_ls_svm}), of which the kernel trick for an asymmetric kernel is crucial. 
\begin{theorem}[]\label{Proposition1}
Let $b_1^\star,b_2^\star,\alphas^\star,\betas^\star$ be the solution of the problem below, where $\bm{H}_{ij}=y_i\phi_s(\bm{x}_i)^\top\phi_t(\bm{x}_j)y_j = y_i \mathcal{K}(\bm{x}_i,\bm{x}_j) y_j$ with an asymmetric kernel $\mathcal{K}$.
    \begin{equation}\label{linear_system}
    \left[
    \begin{array}{cccc}
    0& 0 &\bm{Y}^\top  & 0\\ 
    0& 0 &0  & \bm{Y}^\top\\ 
    \bm{Y}& 0 &\frac{\bm{I}}{\gamma}  & \bm{H}\\ 
    0& \bm{Y} &\bm{H}^\top  &\frac{\bm{I}}{\gamma}\\ 
    \end{array}
    \right ]
    \left[
    \begin{array}{c}
     b_1 \\ 
     b_2\\
     \alphas\\
     \betas\\
    \end{array}
    \right ]
    =
    \left[
    \begin{array}{c}
     0\\
     0\\
     \bm{1}\\
     \bm{1}
    \end{array}
    \right ],
    \end{equation}
    
\begin{enumerate}
    \item $\omegas^\star$ and $\nus^\star$ are spanned by $\{\phi_t(\bm{x}_i)\}_{i=1}^m$ and $\{\phi_s(\bm{x}_i)\}_{i=1}^m$, respectively.
    \[
    \omegas^\star = \sum_{i=1}^m\beta_i^\star y_i\phi_t(\bm{x}_i), \quad \nus^\star =\sum_{i=1}^m\alpha_i^\star y_i\phi_s(\bm{x}_i),
    \]
where ($\omegas^\star,\nus^\star$) is a stationary point of the primal problem (\ref{asy_ls_svm})
\item The primal problem (\ref{asy_ls_svm}) results in two discriminant functions $f_s$ and $f_t$ as follows
\begin{equation}\label{two_clssifiers}
\left\{ \begin{aligned}
& f_s(x)  = \mathcal{K}(\bm{x},\bm{X})(\betas^\star\odot \bm{Y})+b_1^\star\\
& f_t(x)  = \mathcal{K}(\bm{X},\bm{x})(\alphas^\star\odot \bm{Y})+b_2^\star,\\
\end{aligned} \right.
\end{equation}
where $\bm{X}=\{\bm{x}_i\}_{i=1}^m$ is a training set and $\odot$ denotes a element-wise product,

$\mathcal{K}(\bm{x},\bm{X})=[\mathcal{K}(\bm{x}, \bm{x}_1),\cdots,\mathcal{K}(\bm{x}, \bm{x}_m)]$,

$\mathcal{K}(\bm{X},\bm{x})=[\mathcal{K}(\bm{x}_1,\bm{x}),\cdots,\mathcal{K}(\bm{x}_m,\bm{x})]$.

\end{enumerate}    
\end{theorem}
\begin{proof}
The Lagrangian function of the primal problem (\ref{asy_ls_svm}) is formulated as follows:

\begin{equation}
\begin{aligned}\label{Lagrangian}
\mathcal{L}(\Theta;\alphas,\betas) &= \omegas^\top \nus  + \frac{\gamma}{2}\sum_{i=1}^m e_i^2 + \frac{\gamma}{2}\sum_{i=1}^m h_i^2\\
&+\sum_{i=1}^m\alpha_i(1-e_i-y_i(\omegas^\top\phi_s(\bm{x}_i)+b_1))\\
&+\sum_{i=1}^m\beta_i(1-h_i-y_i(\nus^\top \phi_t(\bm{x}_i)+b_2)),
\end{aligned}
\end{equation}
where $\Theta = \{\omegas,\nus,b_1,b_2,\bm{e},\bm{h}\}$ is the parameter set. Then the condition of stationary points requires the following equations:
\begin{equation}\label{partical_derivative}
\left\{ \begin{aligned}
&\frac{\partial \mathcal{L}}{\partial\nus}=\omegas - \sum_{i=1}^m\beta_iy_i\phi_t(\bm{x}_i)=0\\
& \frac{\partial \mathcal{L}}{\partial\omegas} = \nus -\sum_{i=1}^m\alpha_iy_i\phi_s(\bm{x}_i)=0\\
& \frac{\partial \mathcal{L}}{\partial b_1} = \sum_{i=1}^m \alpha_iy_i=0 \\
& \frac{\partial \mathcal{L}}{\partial b_2} = \sum_{i=1}^m \beta_iy_i=0 \\
& \frac{\partial \mathcal{L}}{\partial e_i}=\gamma e_i-\alpha_i=0 \\
& \frac{\partial \mathcal{L}}{\partial h_i}=\gamma h_i-\beta_i =0\\
& \frac{\partial \mathcal{L}}{\partial \alpha_i} = 1-e_i-y_i(\omegas^\top\phi_s(\bm{x}_i)+b_1)=0\\
& \frac{\partial \mathcal{L}}{\partial \beta_i} =1-h_i-y_i(\nus^\top \phi_t(\bm{x}_i)+b_2)=0.
\end{aligned} \right.
\end{equation}
The last two conditions can be converted into the following equations
\begin{equation}\label{primal_equations}
\left\{ \begin{aligned}
& \frac{\partial \mathcal{L}}{\partial \alpha_i} = 1-\frac{\alpha_i}{\gamma}-y_ib_1-y_i\sum_{j=1}^m \beta_jy_j\phi_s^\top(\bm{x}_i)\phi_t(\bm{x}_j) = 0\\
& \frac{\partial \mathcal{L}}{\partial \beta_i} =1-\frac{\beta_i}{\gamma}-y_ib_2-y_i\sum_{j=1}^m \alpha_jy_j\phi_t^\top(\bm{x}_i)\phi_s(\bm{x}_j) = 0.\\
 \end{aligned} \right.
\end{equation}
The equations (\ref{primal_equations}) can be formulated as a linear system as follows,
\begin{equation}\label{linear_system_1}
\left[
\begin{array}{cccc}
0& 0 &\bm{Y}^\top  & 0\\ 
0& 0 &0  & \bm{Y}^\top\\ 
\bm{Y}& 0 &\frac{\bm{I}}{\gamma}  & \bm{Z_sZ_t}^\top\\ 
0& \bm{Y} &\bm{Z_tZ_s}^\top  &\frac{\bm{I}}{\gamma}\\ 
\end{array}
\right ]
\left[
\begin{array}{c}
 b_1 \\ 
 b_2\\
 \alphas\\
 \betas\\
\end{array}
\right ]
=
\left[
\begin{array}{c}
 0\\
 0\\
 \bm{1}\\
 \bm{1}
\end{array}
\right ],
\end{equation}
where
\begin{equation*}
\left\{ \begin{aligned}
& \bm{Z_s}=[y_1\phi_s^\top(\bm{x}_1);\cdots;y_m\phi_s^\top(\bm{x}_m)]\\
& \bm{Z_t}=[y_1\phi_t^\top(\bm{x}_1);\cdots;y_m\phi_t^\top(\bm{x}_m)].\\
 \end{aligned} \right.
\end{equation*}


According to Definition \ref{def:asy_kernel}, an asymmetric kernel is defined as follows:
\[
\mathcal{K}(\bm{x}_i,\bm{x}_j) = \left<\phi_s(\bm{x}_i),\phi_t(\bm{x}_j)\right>.
\]
Then, the linear system (\ref{linear_system_1}) can be reformulated as follows:
\begin{equation*}
    \left[
    \begin{array}{cccc}
    0& 0 &\bm{Y}^\top  & 0\\ 
    0& 0 &0  & \bm{Y}^\top\\ 
    \bm{Y}& 0 &\frac{\bm{I}}{\gamma}  & \bm{H}\\ 
    0& \bm{Y} &\bm{H}^\top  &\frac{\bm{I}}{\gamma}\\ 
    \end{array}
    \right ]
    \left[
    \begin{array}{c}
     b_1 \\ 
     b_2\\
     \alphas\\
     \betas\\
    \end{array}
    \right ]
    =
    \left[
    \begin{array}{c}
     0\\
     0\\
     \bm{1}\\
     \bm{1}
    \end{array}
    \right ],
    \end{equation*}
where $\bm{H}_{ij}=y_i\phi_s(\bm{x}_i)^\top\phi_t(\bm{x}_j)y_j = y_i \mathcal{K}(\bm{x}_i,\bm{x}_j) y_j$ with a given asymmetric kernel $\mathcal{K}$.

Suppose $b_1^\star,b_2^\star,\alphas^\star,\betas^\star$ be the solution of (\ref{linear_system}), according to partial derivative equations (\ref{partical_derivative}), a stationary point ($\omegas,\nus$) can be formulated as below,
\[
\omegas^\star = \sum_{i=1}^m\beta_i^\star y_i\phi_t(\bm{x}_i), \quad \nus^\star =\sum_{i=1}^m\alpha_i^\star y_i\phi_s(\bm{x}_i).
\]

Since a stationary point of the primal problem (\ref{asy_ls_svm}) is obtained, two functions which classify samples from source and target points of view respectively can be formulated as follows:
\[
\left\{ \begin{aligned}
& f_s(\bm{x}) = {\omegas^\star} ^\top\phi_s(\bm{x})+b_1^\star = \mathcal{K}(\bm{x},\bm{X})(\betas^\star\odot \bm{Y})+b_1^\star\\
& f_t(\bm{x}) = {\nus^\star}^\top\phi_t(\bm{x})+b_2^\star = \mathcal{K}(\bm{X},\bm{x})(\alphas^\star\odot \bm{Y})+b_2^\star.\\
\end{aligned} \right.
\]
\end{proof}

The primal-dual relationship between (\ref{asy_ls_svm}) and (\ref{linear_system}) for asymmetric kernels is characterized by Proposition \ref{Proposition1}, which also includes symmetric kernels. In that case, both the primal and dual formulations reduce to the classical LS-SVM, as shown below. 

\begin{theorem}[]\label{Proposition2}
When the kernel $\mathcal{K}$ in (\ref{asy_ls_svm}) is symmetric,
\begin{enumerate}
    \item two functions $f_s$ and $f_t$ are identical;
    \item two linear systems (\ref{linear_system}) and (\ref{linear_system_ls_svm}) are equivalent.
\end{enumerate}
\end{theorem}
\begin{proof}
According to Definition \ref{def:asy_kernel}, symmetric kernels can be also defined in the asymmetric kernel framework. Thus, kernels in AsK-LS can be also positive semi-definite, even indefinite.

A solution to the primal problem (\ref{asy_ls_svm}) is given by the linear system (\ref{linear_system}) which can be reformulated as follows, according to the matrix column transformation ($\rm CT$) and row transformation ($\rm RT$) formulas,

\begin{normalsize} 
\begin{equation*}
\begin{aligned}
\left[
\begin{array}{cccc}
0& 0 &\bm{Y}^\top  & 0\\ 
0& 0 &0  & \bm{Y}^\top\\ 
\bm{Y}& 0 &\frac{\bm{I}}{\gamma}  & \bm{H}\\ 
0& \bm{Y} &\bm{H}^\top  &\frac{\bm{I}}{\gamma}\\ 
\end{array}
\right]
&\left[
\begin{array}{c}
 b_1 \\ 
 b_2\\
 \alphas\\
 \betas\\
\end{array}
\right ]
=
\left[
\begin{array}{c}
 0\\
 0\\
 \bm{1}\\
 \bm{1}
\end{array}
\right ]\\
=&\left[
\begin{array}{cccc}
0& 0 &\bm{Y}^\top  & 0\\ 
0& 0 &0  & \bm{Y}^\top\\ 
\bm{Y}& 0 &\frac{\bm{I}}{\gamma}  & \bm{H}^\top\\ 
0& \bm{Y} &\bm{H}  &\frac{\bm{I}}{\gamma}\\ 
\end{array}
\right ]
\left[
\begin{array}{c}
 b_2 \\ 
 b_1\\
 \betas\\
 \alphas\\
\end{array}
\right ].
\end{aligned}
\end{equation*}
\end{normalsize}


Feeding the symmetric kernel $\mathcal{K}$, $\bm{H}$ is then a symmetric matrix, $\bm{H}_{ij}=y_i\phi_s(\bm{x}_i)^\top\phi_t(\bm{x}_j)y_j = y_i \mathcal{K}(\bm{x}_i,\bm{x}_j) y_j$, i.e., $\bm{H}=\bm{H}^\top$. The following equation holds,


\begin{equation*}\label{equation1}
\begin{aligned}
    \left[
    \begin{array}{cccc}
    0& 0 &\bm{Y}^\top  & 0\\ 
    0& 0 &0  & \bm{Y}^\top\\ 
    \bm{Y}& 0 &\frac{\bm{I}}{\gamma}  & \bm{H}\\ 
    0& \bm{Y} &\bm{H}^\top  &\frac{\bm{I}}{\gamma}\\ 
    \end{array}
    \right]
    &\left[
    \begin{array}{c}
     b_1 \\ 
     b_2\\
     \alphas\\
     \betas\\
    \end{array}
    \right ]
    \\
=&\left[
    \begin{array}{cccc}
    0& 0 &\bm{Y}^\top  & 0\\ 
    0& 0 &0  & \bm{Y}^\top\\ 
    \bm{Y}& 0 &\frac{\bm{I}}{\gamma}  & \bm{H}\\ 
    0& \bm{Y} &\bm{H}^\top  &\frac{\bm{I}}{\gamma}\\ 
    \end{array}
    \right]
    \left[
    \begin{array}{c}
     b_2 \\ 
     b_1\\
     \betas\\
     \alphas\\
    \end{array}
    \right ].
    \end{aligned}
\end{equation*}

The above equation can be simplified by moving the right term to the left.

\begin{equation}\label{equation2}
\begin{aligned}
    &\left[
    \begin{array}{cccc}
    0& 0 &\bm{Y}^\top  & 0\\ 
    0& 0 &0  & \bm{Y}^\top\\ 
    \bm{Y}& 0 &\frac{\bm{I}}{\gamma}  & \bm{H}\\ 
    0& \bm{Y} &\bm{H}^\top  &\frac{\bm{I}}{\gamma}\\ 
    \end{array}
    \right]
    \left[
    \begin{array}{c}
     b_1 -b_2\\ 
     b_2 -b_1\\
     \alphas-\betas\\
     \betas-\alphas\\
    \end{array}
    \right ]
    =\bm{0}\\ 
    \stackrel{\rm{RT}}{\longrightarrow}
    &\left[
    \begin{array}{cccc}
    \bm{Y}& 0 &\frac{\bm{I}}{\gamma}  & \bm{H}\\ 
    0& \bm{Y} &\bm{H}^\top  &\frac{\bm{I}}{\gamma}\\ 
    0& 0 &\bm{Y}^\top  & 0\\ 
    0& 0 &0  & \bm{Y}^\top\\
    \end{array}
    \right ]
    \left[
    \begin{array}{c}
     b_1 -b_2\\ 
     b_2 -b_1\\
     \alphas-\betas\\
     \betas-\alphas\\
    \end{array}
    \right ] \\
    \triangleq&\bm{A}\cdot\left[
    \begin{array}{c}
     b_1 -b_2\\ 
     b_2 -b_1\\
     \alphas-\betas\\
     \betas-\alphas\\
    \end{array}
    \right ]
    =\bm{0}, 
    \end{aligned}
\end{equation}
where the matrix $\bm{A}\in\mathbb{R}^{(2m+2)\times(2m+2)}$ denotes the system matrix. It can be noticed that $\bm{A}$ has a full rank and the equation (\ref{equation2}) indicates that $b_1 = b_2$ and $\alphas=\betas$. Then two functions in (\ref{two_clssifiers}) are identical as below,
\begin{equation*}
\begin{aligned}
    f_s(\bm{x}) &= \mathcal{K}(\bm{x},\bm{X})(\betas^\star\odot \bm{Y})+b_1^\star\\
    &= \mathcal{K}(\bm{X},\bm{x})(\alphas^\star\odot \bm{Y})+b_2^\star = f_t(\bm{x}).
\end{aligned}
\end{equation*}

Since $b_1 = b_2$, $\alphas=\betas$ and $\bm{H}$ is a symmetric matrix, the linear system (\ref{linear_system}) can be simplified into a lower-dimensional linear system as follows:
\begin{equation*}
    \left[
    \begin{array}{cc}
    0& \bm{Y}^\top\\ 
    \bm{Y} &\frac{\bm{I}}{\gamma}+\bm{H}\\ 
    \end{array}
    \right ]
    \left[
    \begin{array}{c}
     b_1 \\ 
     \alphas\\
    \end{array}
    \right ]
    =
    \left[
    \begin{array}{c}
     0 \\ 
     \bm{1}\\
    \end{array}
    \right ],
\end{equation*}
which is equivalent to (\ref{linear_system_ls_svm}) and we complete the proof.
\end{proof}

Symmetric and asymmetric kernels lead to a unified linear system. When the data size is not very large, (\ref{linear_system}) is easy to solve. For large-scale problems, one can consider the fixed-size \cite{espinoza2006fixed,de2010optimized}, Nystr\"{o}m approximation \cite{williams2001using}, for the latter of which necessary modifications are required. In the early work by Schmidt \cite{stewart1993early}, an asymmetric kernel has a pair of adjoint eigenfunctions corresponding to the eigenvalue, thus the modified Nystr\"{o}m approximation for asymmetric kernels leads to a standard singular value decomposition.
When the kernel is asymmetric, we simultaneously obtain two functions $f_s,f_t$. To fully use the asymmetric information, we need to merge them together. Averaging is a simple way and other ensemble methods like the stacked generalization \cite{stacked_generalization} can be also used. Overall, we summarize the algorithm for AsK-LS in the following.

\begin{algorithm}[htb]
\renewcommand{\algorithmicrequire}{\textbf{Input:}}
\renewcommand{\algorithmicensure}{\textbf{Output:}}
\caption{Learning with an asymmetric kernel in AsK-LS}
\label{alg:A}  
\begin{algorithmic}[1]
\REQUIRE The asymmetric kernel $\mathcal{K}$, the regularization parameter $\gamma$, training data $(\bm{X},\bm{Y})$, and testing data $\bm{Z}$.\\
\ENSURE The prediction $f(\bm{Z})$ of testing data $\bm{Z}$.
\STATE Calculate an asymmetric kernel matrix $\bm{H}$ by the asymmetric kernel $\mathcal{K}$ and training data $(\bm{X},\bm{Y})$.
\STATE $[\alphas^\star,\betas^\star,b_1^\star,b_2^\star]\leftarrow$ solve the linear system  (\ref{linear_system}).
\STATE Predict testing data from source and target views, i.e., $f_s(\bm{Z})$ and $f_t(\bm{Z})$.
\STATE Merge two classifiers $f_s(\bm{Z})$ and $f_t(\bm{Z})$ to obtain the final prediction $f(\bm{Z})$.
\RETURN $f(\bm{Z})$.
\end{algorithmic}  
\end{algorithm}

\section{Numerical experiments}
\label{sec:experiment}
The aim of designing the AsK-LS is to learn with asymmetric kernels. As discussed before, asymmetric kernels could come from asymmetric metrics and directed graphs. The following experiments are not to claim that asymmetric kernels are better than symmetric kernels, which is surely not true since the choice of kernels is problem-dependent. Instead, we will show that when the asymmetric information is crucial, our AsK-LS can largely improve the performance of the existing kernel learning methods. The experiments are implemented in MATLAB on a PC with Intel i7-10700K CPU (3.8GHz) and 32GB memory. All the reported results are the average over $10$ trials.

\subsection{Kullback-Leibler kernel}
Kullback-Leibler divergence is the measure of how one probability distribution $Q$ is different from another probability distribution $P$:
\begin{equation}\label{kl}
    \rm KL( P || Q ) = \int_{-\infty}^\infty P(x) \log \frac{P(x)}{Q(x)}dx,
\end{equation}
which is asymmetric. KL divergence has been successfully used in many fields, e.g., VAE \cite{Kingma2014} and GAN \cite{karras2019style,gui2021review}. 
From KL divergence, one can evaluate the similarity between two probability distributions by the exponentiation as follows:
\begin{equation}\label{kl_kernel}
    \mathcal{K}(s_i, s_j) = \exp(-a\cdot\rm KL( P_i || P_j )),
\end{equation}
where $a\in\mathbb{R}^+$ is a hyperparameter to control the scale of the kernel (\ref{kl_kernel}), $s$ denotes a sample and $P_i$,$P_j$ are probability distributions estimated by samples $s_i$ and $s_j$, respectively. Since KL divergence is asymmetric, the kernel matrix is also asymmetric. Thus the AsK-LS can be utilized to directly learn with the asymmetric KL kernel rather than its symmetry \cite{moreno2003kullback}.

We conduct image classification experiments on the Corel database \cite{wang2001simplicity}. The Corel database contains 80 concept groups e.g., autumn, aviation, dog, and elephant. And 10 classes of those groups are picked: beaches, bus, dinosaurs, elephants, flowers, foods, horses, monuments, snow mountains, and African people \& villages. There are 100 images per class: 90 for training and 10 for testing. We follow the standard feature extraction method \cite{moreno2003kullback} to obtain a sequence of 64-dimensional discrete cosine transform feature vectors $\bm{X} = \{\bm{x}_1,\bm{x}_2,\cdots,\bm{x}_l\}$ of an image where $l$ is the number of feature vectors.

In the experiment, we use the Gaussian mixture model (GMM) with 256 diagonal Gaussian components to estimate the probability distribution $P_i(\bm{X}_i)$ of the image feature vector sequence $\bm{X}_i$. Since the KL divergence for two GMMs can not be calculated by the formulation (\ref{kl}) directly, we use the Monte Carlo method to calculate it. Then kernel value between two images is given by the KL kernel similarity (\ref{kl_kernel}) of these two probability distributions.

The corresponding asymmetric KL divergence has been widely used in learning tasks. However, due to that KL divergence violates the symmetry requirement, it has never been directly used in the kernel-based learning. Now with the AsK-LS, we can use it in classification tasks. As a comparison, the SVM and the LS-SVM are used with a symmetric KL kernel \cite{moreno2003kullback}. For the parameters, i.e., the hyperparameter $a$ and regularization parameter $\gamma$ in all the three models is tuned by 10-folds cross-validation. The AsK-LS outputs two classifiers and here we simply average them.

The image classification is a multi-classes task where we utilize the one-vs-rest scheme, for which the average Micro-F1 and Macro-F1 scores are reported in Table \ref{tab:1}. The AsK-LS achieves better performance, showing that learning with the asymmetric information can help.


\begin{table}[]
\centering
\setlength{\abovecaptionskip}{0pt}%
\setlength{\belowcaptionskip}{5pt}%
\caption{Micro-/Macro-F1 scores (mean±std) of different methods with the KL kernel between GMMs on the Corel dataset.}
\begin{tabular}{cccc}
\specialrule{0.05em}{3pt}{3pt}
Methods & SVM & LS-SVM & AsK-LS\\
\specialrule{0.05em}{3pt}{3pt}
Micro-F1 & 0.830±0 &0.840±0 & \textbf{0.916±0.0048} \\
Macro-F1 & 0.822±0 & 0.832±0 &\textbf{0.914±0.0046}  \\
\specialrule{0.05em}{3pt}{3pt}
\end{tabular}
\label{tab:1}
\end{table}

\begin{table}[ht]
\normalsize
\centering
\setlength{\abovecaptionskip}{0pt}%
\setlength{\belowcaptionskip}{5pt}%
\caption{The detailed information of used directed graph datasets.}
\begin{tabular}{cccccc}
\specialrule{0.05em}{3pt}{3pt}
\multirow{2}{*}{Datasets} & \multirow{2}{*}{Cora} & \multirow{2}{*}{Citeseer} & \multirow{2}{*}{Pubmed} & AM- & AM-\\ 
&&&&photo&computer\\
\specialrule{0.05em}{3pt}{3pt}
\#Classes & 7 & 6 & 3 & 8 & 10  \\
\#Nodes & 2078 & 3327 & 19717 & 7650 & 13752 \\
\#Edges & 5429 & 4732& 44338 &143663 & 287209  \\ \specialrule{0.05em}{3pt}{3pt}
\end{tabular}
\label{tab:2}
\end{table}
\normalsize

\subsection{Directed adjacency matrix}

\begin{table*}[ht]
\normalsize
\centering
\setlength{\abovecaptionskip}{0pt}%
\setlength{\belowcaptionskip}{5pt}%
\caption{Micro-/Macro-F1 scores (mean±std) of different algorithms on the nodes classification task.}
\begin{tabular}{ccccccc}
\specialrule{0.05em}{3pt}{3pt}
\multirow{2}{*}{Datasets}  & \multirow{2}{*}{F1 Score} & \multicolumn{2}{c}{Embedding} & \multicolumn{3}{c}{Graph} \\
\cmidrule(r){3-4} \cmidrule(r){5-7}
 && SVM & LS-SVM & SVM & LS-SVM & AsK-LS  \\ \specialrule{0.05em}{3pt}{3pt}
\multirow{2}{*}{Cora} & Micro &0.740±0.010
 & 0.733±0.009& 0.305±0.010 & 0.713±0.015 & \textbf{0.753±0.013} \\
& Macro&0.730±0.012 &0.724±0.011  &0.077±0.004 & 0.708±0.016& \textbf{0.748±0.013} \\\specialrule{0.05em}{3pt}{3pt}
\multirow{2}{*}{Citeseer}& Micro& 0.497±0.013& 0.507±0.010 &0.393±0.042 & 0.596±0.011 & \textbf{0.605±0.012} \\
&Macro&0.454±0.013 & 0.452±0.010  & 0.301±0.044& 0.560±0.012 & \textbf{0.579±0.012} \\\specialrule{0.05em}{3pt}{3pt}
\multirow{2}{*}{Pubmed} & Micro &\textbf{0.732±0.006} &0.726±0.004  &0.602±0.022 & 0.635±0.006 & 0.719±0.004  \\
&Macro &0.680±0.006 & 0.669±0.005 &0.451±0.018 & 0.503±0.006 & \textbf{0.698±0.005} \\\specialrule{0.05em}{3pt}{3pt}
AM-& Micro&0.858±0.005 &0.834±0.006 &0.319±0.022 & 0.799±0.008 &\textbf{0.885±0.005}  \\
photo&Macro&0.834±0.005 &0.767±0.006 & 0.155±0.018 &0.732±0.008 &\textbf{0.874±0.005} \\\specialrule{0.05em}{3pt}{3pt}
AM- & Micro& 0.606±0.005&0.609±0.004 & 0.400±0.005 & 0.619±0.014 &\textbf{0.868±0.002}  \\
computer&Macro&0.429±0.004 &0.454±0.003  &0.107±0.003 &0.443±0.016&\textbf{0.846±0.005} \\\specialrule{0.05em}{3pt}{3pt}
\end{tabular}
\label{tab:3}
\end{table*}

Nodes classification with an asymmetric adjacency matrix is a task that needs to learn from asymmetric metrics. 
In this task, nodes in a directed graph $\mathcal{H}=\{\mathcal{N},\mathcal{E}\}$ with nodes set $\mathcal{N}$ and edges set $\mathcal{E}$ are classified. Its adjacency matrix showing the connection among nodes is defined as follows,
\[
\mathcal{A}_{ij}=\left\{ \begin{aligned}
& 1 \quad \textrm{if~} j \rightarrow i\\
& 0 \quad \textrm{otherwise}.\\
\end{aligned} \right.
\]
where $j \rightarrow i$ means that there is a link pointing to node $i$ from node $j$. This model has wide applications and here we consider five directed graphs, namely Cora, Citeseer, and Pubmed \cite{sen2008collective}, AM-photo, and AM-computer \cite{shchur2018pitfalls}. Details of the data set could be found in Table \ref{tab:2}. For the first three widely used graphs, the nodes and edges present documents and citations, respectively. For the latter two, 
the nodes present goods, and the edges mean the two goods are frequently bought together. 
The node classification 
is originally a multi-classes task where we utilize the one-vs-rest scheme and focus on Micro-F1 and Macro-F1 scores.

A directed graph is fully characterized by its adjacency matrix. However, in existing kernel methods, one cannot directly use the adjacency matrix as a kernel. 
Therefore, the current mainstream is to first do feature embedding \cite{ou2016asymmetric, zhou2017scalable, khosla2019node} 
to extract the asymmetric information and then do classification based on the extracted features. 
In the experiment, 
For the embedding, we use a SOTA embedding method NERD \cite{khosla2019node}, which utilizes random walk to extract node embeddings $\phi_s(v)$ and $\phi_t(v)$ as source feature and target feature of each node $v\in \mathcal{N}$ on a directed graph. We combine them as a unified feature $\phi(v) = [\phi^\top_s(v), \phi^\top_t(v)]^\top$ which then defines a symmetric kernel that can be used in classical kernel-based learning methods, e.g., the SVM and the  LS-SVM, and results are reported in the third and the fourth columns in Table \ref{tab:3}. 

Since the asymmetric information can be extracted by the embedding method, the performance of that is much better than using existing kernel methods with the adjacency matrix by symmetrization $(\mathcal{A}+\mathcal{A}^\top)/2$, of which micro-/macro- F1 scores are listed in the 5th and the 6th column in Table \ref{tab:3}.


Now we have the AsK-LS and can directly learn with the adjacency matrix without the feature embedding. Before sending the adjacency matrix to the AsK-LS, we pre-process it by its in-degree $d_i = \sum_{j=1}^m \mathcal{A}_{ij}$ (the same pre-process is also applied for the SVM and the LS-SVM). The Micro- and Macro-F1 scores of the AsK-LS are reported in the last column in Table \ref{tab:3}. The performance is much better than the SVM and the LS-SVM with the symmetrization of the kernel, indicating that the asymmetric information is helpful. For the comparison with the SOTA embedding methods, the AsK-LS is better or at least comparable, showing the effectiveness of AsK-LS for extracting the asymmetric information.

\subsection{Symmetric or asymmetric kernels}
We have shown that the proposed AsK-LS can learn with asymmetric kernels. Since asymmetric kernels are more general than symmetric ones, learning with asymmetric kernels is promising to get improvement, if there is an efficient method to get a suitable kernel. In the case in our paper, kernels are pre-given, then the performance is determined by the choice of kernels. In the previous sections, 
the asymmetric information, i.e., measuring the difference by KL divergence and directed distance, is important, thus the corresponding asymmetric kernels lead to a good performance. 
If not the case, a specific asymmetric kernel is not necessarily better than the symmetric one, e.g., the RBF kernel. 

We conduct classification experiments on several datasets from the UCI database\cite{uci}, where 60\% of the data are randomly picked up for training and the rest for testing. 
Two asymmetric kernels, called \emph{SNE kernel} and \emph{T kernel}, are considered here. They have been used for dimension reduction \cite{hinton2002stochastic} but have not been used for classification, since before there was no classification method that can learn with asymmetric kernels directly. The formulations are given as below,
\begin{enumerate}
    \item The SNE kernel with the parameter $\sigma\in\mathbb{R}$:
    \[
    \mathcal{K}_{\rm{SNE}}(\bm{x},\bm{y}) = \frac{\exp(-\|\bm{x}-\bm{y}\|_2^2)/\sigma^2}{\sum_{\bm{z}\in \bm{X}}\exp(-\|\bm{x}-\bm{z}\|_2^2)/\sigma^2},
    \]
    \item The T kernel:
    \[
    \mathcal{K}_{\rm{T}}(\bm{x},\bm{y}) = \frac{(1+\|\bm{x}-\bm{y}\|^2_2)^{-1}}{\sum_{\bm{z}\in \bm{X}}(1+\|\bm{x}-\bm{z}\|^2_2)^{-1}},
    \]
\end{enumerate}
where $\bm{X}$ stands for the training set.

\begin{table}[h]
\normalsize
\centering
\setlength{\abovecaptionskip}{0pt}%
\setlength{\belowcaptionskip}{5pt}%
\caption{Classification accuracy (mean±std) of the LS-SVM and the AsK-LS with RBF, SNE, and T kernels on the UCI database.}
\begin{tabular}{cccc}
\specialrule{0.05em}{3pt}{3pt}
\multirow{2}{*}{Datasets}& LS-SVM & \multicolumn{2}{c}{AsK-LS} \\
\cmidrule(r){2-2}\cmidrule(r){3-4}
& RBF &SNE &T\\
\specialrule{0.05em}{3pt}{3pt}
heart  &\textbf{0.837±0.032}  &0.824±0.025
&0.825±0.031  \\
sonar  &0.856±0.001   &0.854±0.028
 &\textbf{0.865±0.027} \\
monks1  &\textbf{0.791±0.012} &0.790±0.014
 &0.778±0.012 \\
monks2  &0.841±0.007    &\textbf{0.866±0.016}
& \textbf{0.866±0.014}\\
monks3  &\textbf{0.936±0.003}  &\textbf{0.936±0.004}
&0.901±0.004\\
pima &0.738±0.026   &0.749±0.026
&\textbf{0.752±0.021}\\
australian  &\textbf{0.862±0.015}  &0.854±0.019
 &0.859±0.032\\
spambase  &\textbf{0.925±0.007}  &0.908±0.007
 &0.922±0.018\\
\specialrule{0.05em}{3pt}{3pt}
\end{tabular}
\label{tab:4}
\end{table}

The classification accuracy of the AsK-LS with the two asymmetric kernels is reported in Table \ref{tab:4}, together with the accuracy of the LS-SVM with the RBF kernel (all the parameters are tuned by 10-folds cross-validation). Generally speaking, the best choice of kernels is problem-dependent and one cannot assert which kernel is good in advance. But at least, the proposed AsK-LS makes it possible to use asymmetric kernels. With delicately designing or efficiently learning, asymmetric kernels could lead to a good performance.

\section{Conclusion}
\label{sec:conclusion}
In this paper, we investigate the least squares support vector machine with asymmetric kernels in theoretical and algorithmic aspects. The proposed AsK-LS is the first model that can learn with asymmetric kernels. The primal and dual representations for AsK-LS are given, showing the feature interpretation that there are two different functions, can be simultaneously learned from the source and the target views. In numerical experiments, when the asymmetric information is physically important, the AsK-LS with asymmetric kernels significantly outperforms the SVM and the LS-SVM that can only deal with symmetric kernels. 

The most significant contribution of this paper is to make asymmetric kernels useful in the kernel-based learning. In methodology, the least squares framework is not the unique way to accommodate asymmetric kernels. Models from other kernel-based methods, e.g., the support vector machine and the logistic regression, etc., are worthy to be investigated. In theory, the functional space associated with asymmetric kernels is interesting, which is beyond the Reproducing Kernel Hilbert Space for PSD kernels or the Reproducing Kernel {Kre\u\i n} Space for indefinite kernels. In application, one can design asymmetric kernels for different tasks, especially those that involve directional relationships, including but not limited to directed graphs, the distribution distance \cite{zhou2006sample}, the causality analysis \cite{radinsky2012learning,xu2016learning}, and the optimal transport \cite{courty2016optimal,su2018order}.

\ifCLASSOPTIONcompsoc
  \section*{Acknowledgments}
\else
  \section*{Acknowledgment}
\fi
This work was supported by National Natural Science Foundation of China (No. 61977046), Shanghai Municipal Science and Technology Major Project (2021SHZDZX0102) and Shanghai Science and Technology Research Program (20JC1412700 and 19JC1420101). The research leading to these results has received funding from the European Research Council under the European Union’s Horizon 2020 research and innovation program / ERC Advanced Grant E-DUALITY (787960). This paper reflects only the authors’ views and the Union is not liable for any use that may be made of the contained information. This work was supported in part by Research Council KU Leuven: Optimization frameworks for deep kernel machines C14/18/068; Flemish Government: FWO projects: GOA4917N (Deep Restricted Kernel Machines: Methods and Foundations), PhD/Postdoc grant. This research received funding from the Flemish Government (AI Research Program).


\ifCLASSOPTIONcaptionsoff
  \newpage
\fi





\input{main.bbl}

\bibliographystyle{IEEEtran}

\end{document}